\documentclass[conference]{IEEEtran}
\IEEEoverridecommandlockouts
\usepackage{cite}
\usepackage{amsmath,amssymb,amsfonts}
\usepackage{algorithmic}
\usepackage{graphicx}
\usepackage{textcomp}
\usepackage{xcolor}
\usepackage{marvosym}

\usepackage{ieee-con-nlw}

\usepackage{booktabs}
\usepackage{makecell}
\usepackage{CJKutf8}

\begin{document}
\begin{CJK}{UTF8}{gbsn}

\title{Boolean-aware Boolean Circuit Classification: A Comprehensive Study on Graph Neural Network}

\author{ Liwei~Ni$^{1,2,3}$, Xinquan~Li$^{2}$, Biwei~Xie$^{1,2,3}$ and Huawei~Li$^{1,2,3}$ \\
$^1$Institute of Computing Technology, Chinese Academy of Sciences, Beijing, China \\
$^2$Peng Cheng Laboratory, Shenzhen, China \\
$^3$University of Chinese Academy of Sciences, Beijing, China \\
Emails: nlwmode@gmail.com
}

\maketitle

\begin{abstract}
Boolean circuit is a computational graph that consists of the dynamic directed graph structure and static functionality.
The commonly used logic optimization and Boolean matching-based transformation can change the behavior of the Boolean circuit for its graph structure and functionality in logic synthesis.
The graph structure-based Boolean circuit classification can be grouped into the graph classification task, however, the functionality-based Boolean circuit classification remains an open problem for further research.
In this paper, we first define the proposed \textit{matching-equivalent class} based on its ``Boolean-aware'' property.
The Boolean circuits in the proposed class can be transformed into each other.
Then, we present a commonly study framework based on graph neural network~(GNN) to analyze the key factors that can affect the Boolean-aware Boolean circuit classification.
The empirical experiment results verify the proposed analysis, and it also shows the direction and opportunity to improve the proposed problem.
The code and dataset will be released after acceptance.
\end{abstract}





\section{Introduction}




Boolean circuit~\cite{boolean_circuit_spring10} is a computational graph that consists of two parts: the dynamic directed graph structure~(DAG), and the static functionality.
The Boolean circuit classification is an important problem in the logic synthesis domain.
It can enhance many tasks, like logic optimization~\cite{drills}, technology mapping~\cite{lsoracle, liu2023aimap}, formal verification~\cite{NP3}, reverse engineering~\cite{reverse_engineering}, \textit{etc}.
From the component of the Boolean circuit, the geometric structure and functionality are the two aspects focused on in the downstream applications and tasks, also fit for the Boolean circuit classification problem.

For the geometric structure aspect, the Boolean circuit classification problem focuses on the graph structure of the Boolean circuit. 
Hence, it can be classified into the common graph classification problem\cite{classification_survery}, and this topic is also mature enough by the deep learning approaches.
However, for the functionality aspect, the ``Boolean-aware Boolean circuit classification'' is still an open problem for further study.
It requires that the functionality is the key point for consideration of the Boolean-aware circuit classification.
The main problems are the definition of ``Boolean-aware'' and its corresponding class.

Logic optimization and Boolean matching are the commonly used Boolean transformations in logic synthesis~\cite{nutshell}.
Logic optimization~\cite{rewrite, refactor, balance} can change the Boolean circuit structure while retaining the functionality.
The Boolean matching~\cite{boolean_matching} can modify the Boolean function of Boolean circuits while an invertible operation can recover its Boolean function.
Thus, we define the ``Boolean-aware'' that involves logic optimization and Boolean matching, and the ``matching-equivalent class'' for their operations as shown in the \cref{fig:boolean_transformation}.
The logic optimization operations generate Boolean circuits in a logic-equivalent class, and Boolean matching can modify Boolean circuits into other logic-equivalent classes. 
Then these different logic-equivalent classes can be merged into the defined matching-equivalent class by the invertible laws.


\begin{figure}
    \centering
    \includegraphics[width=0.48\textwidth]{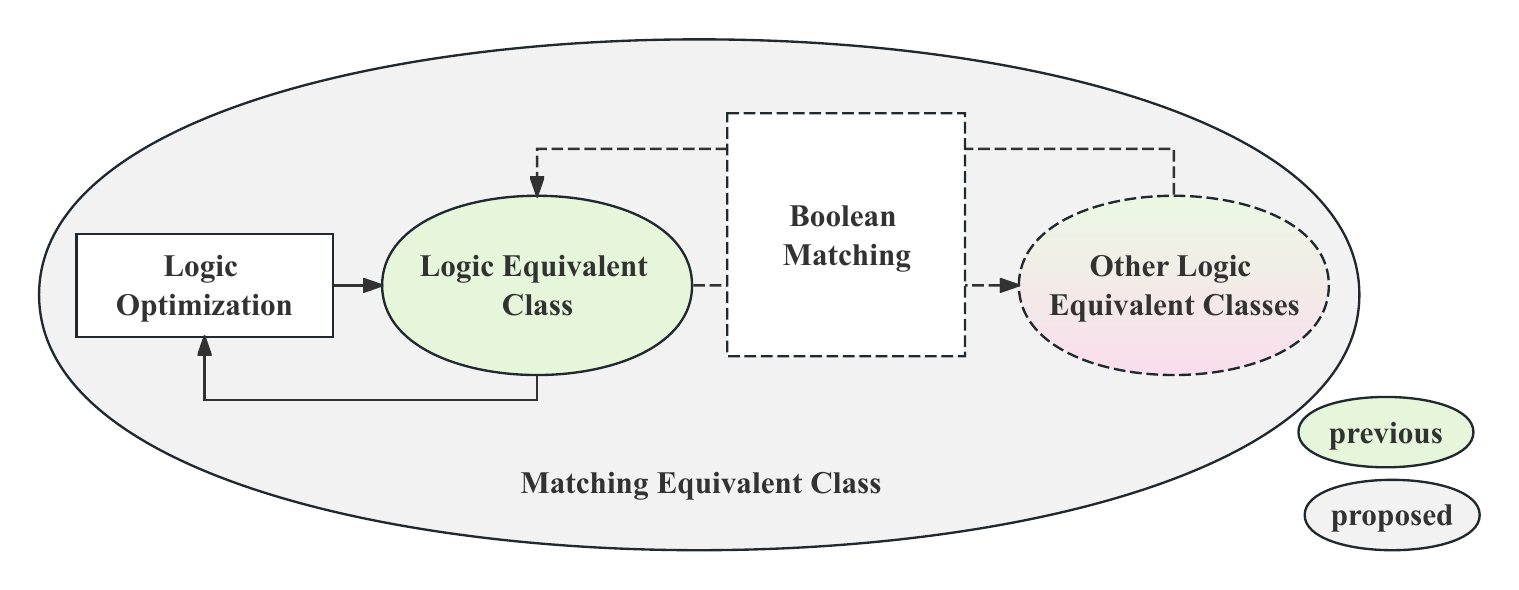}
    \caption{The illustration of the Boolean transformations of logic optimization and Boolean matching. The previous logic equivalent class generated by logic optimization can be transformed into the ``other logic equivalent class'' by Boolean matching operations. All these transformations bring about the proposed matching equivalent class.}
    \label{fig:boolean_transformation}
\end{figure}

After defining the matching-equivalent class based on the Boolean-aware property, we also give a comprehensive study framework based on GNN.
It mainly consists of two parts: the preprocessing phase, and the GNN-based graph embedding and classification phase.
The preprocessing phase shows that even simple modification can improve the accuracy.
The GNN phase mainly focuses on how the node/graph embedding affects the classification results of the Boolean transformation operations.
The empirical experiment results show the direction and opportunity to improve the proposed problem.
The main contributions of this paper are as follows:
\begin{itemize}
    \item We propose the \textit{matching equivalent class} and its corresponding classification problem based on the Boolean-aware property.
    \item We give plenty of theorem and analysis on the properties of matching equivalent class.
    \item We present a commonly study framework based on GNN.
    \item The experimental results echo the theorem and analysis, which also show the research direction in this task.
\end{itemize}

The rest of the paper is organized as follows: 
\cref{sec:preliminary} presents the preliminaries and background.
\cref{sec:problem} gives the definition of Boolean-aware Boolean circuit classification.
\cref{sec:study} proposes the study framework.
\cref{sec:experiment} shows a series of ablation experiments.
\cref{sec:application} shows the flows of application domains.
Finally, \cref{sec:conclusion} concludes the paper.

\section{Preliminary}
\label{sec:preliminary}
\subsection{Notations}
\paragraph{(Vectorial) Boolean Function.}
A Boolean function~\cite{boolean_function_Donnell21} is a function that operates on one or more Boolean values~(\{0,1\}) and produces a single Boolean value as its output. 
Let $n,m \in \mathbb{N}$. A function $\mathbb{F}_2^{n} \rightarrow \mathbb{F}_2$, with
$$
(x_1, ..., x_n) \mapsto f(x_1, ..., x_n)
$$
is called a \textit{Boolean function}, denoted as $f$.
The set of \textit{Boolean variables} can be denoted as a coordinated vector $\Vec{x} = (x_1, ..., x_n)$ with $n$ Boolean variables. 
All the Boolean variables and Boolean functions take the Boolean value from the Boolean domain $\mathbb{B} = \{0, 1\}$.
Similarly, a function $\mathbb{F}_2^{n} \rightarrow \mathbb{F}_2^{m}$ with 
$$
(x_1, ..., x_n) \mapsto \big{(}f_{1}(x_1, ..., x_n), ..., f_{m}(x_1, ..., x_n)\big{)}
$$
is called \textit{Vectorial Boolean function}, denoted as $\Vec{f}$. 
And the function $f_i:\mathbb{F}_2^{n} \rightarrow \mathbb{F}_2$ are also called the \textit{coordinate functions} of $\Vec{f}$.
The vectorial Boolean function $\Vec{f}$ is degenerated into a Boolean function $f$ iff $m = 1$.

\paragraph{Boolean Circuit.}
A Boolean circuit is defined by the logic gates it contains and consists of two parts: the DAG and the associated Boolean function.
In the following, we will use the Boolean function to refer to either the Boolean function or vectorial Boolean function.
Let $\mathcal{C}=(\mathcal{V}, \mathcal{E})$ denote as a Boolean circuit with nodes set $\mathcal{V}$ and edges set $\mathcal{E}$, where $\mathcal{V} = \mathcal{V}^{PI} \cup \mathcal{V}^{G} \cup \mathcal{V}^{PO}$, and $\mathcal{V}^{PI}$ is the set of primary input nodes (PIs), $\mathcal{V}^{PO}$ is the set of primary output nodes (POs), and $\mathcal{V}^{G}$ is the set of internal logic gates. 
Each edge $v_i \to v_j$ in $\mathcal{E}$ represents the computational dependency relationship between different nodes.
Also, as for the DAG with $n$ nodes, $\mathbi{A}^{n\times n}$ is defined as the Boolean adjacent matrix that $\mathbi{A}_{i,j} = true$ means that there exists an edge $v_i \to v_j$.
For convenience, $\mathcal{C}^{G}$ is defined as the DAG structure of $\mathcal{C}$, meanwhile, $\mathcal{C}^{F}$ is defined as its (vectorial) Boolean function.

\subsection{Boolean Transformation}

\subsubsection{Transformation of Logic Optimization.}

\begin{definition}[logic/Boolean equivalent]
Given a pair of Boolean circuits ($\mathcal{C}_1$, $\mathcal{C}_2$), we can say that $\mathcal{C}_1$ is logic equivalent to $\mathcal{C}_2$, denoted as $\mathcal{C}_1 \equiv \mathcal{C}_2$, in the following three conditions:
\begin{equation}
\nonumber
\begin{aligned}
&\big{(}\mathcal{C}_1^{F} \equiv \mathcal{C}_2^{F}\big{)} \Rightarrow(\mathcal{C}_1 \equiv \mathcal{C}_2);~~
\big{(}\mathcal{C}_1^{G} \equiv \mathcal{C}_2^{G}\big{)} \Rightarrow(\mathcal{C}_1 \equiv \mathcal{C}_2); \\
&\big{(} (\mathcal{C}_1^{G} \neq   \mathcal{C}_2^{G}) \land (\mathcal{C}_1^{F} \equiv \mathcal{C}_2^{F}) \big{)}  \Rightarrow~~(\mathcal{C}_1 \equiv \mathcal{C}_2).
\end{aligned}
\end{equation}
\label{def:logic_equivalence}
\vspace{-1.5em}
\end{definition}

Logic optimization is an important step in the logic synthesis~\cite{nutshell} which aims to improve the circuit's efficiency and performance.
And the commonly used operations in logic optimization are under the constraint of Boolean equivalence as defined above.
Thus, Boolean equivalence~\cite{boolean_equivalence_iccad02} asserts that Boolean circuits with the same Boolean function can have different computational graph structures, which can pose big challenges for the classification tasks.
And \cref{fig:full_adder} (a)$\to$(b) shows the result of the logic optimization transformation.

\begin{figure}[t]
    \centering
    \includegraphics[width=0.45\textwidth]{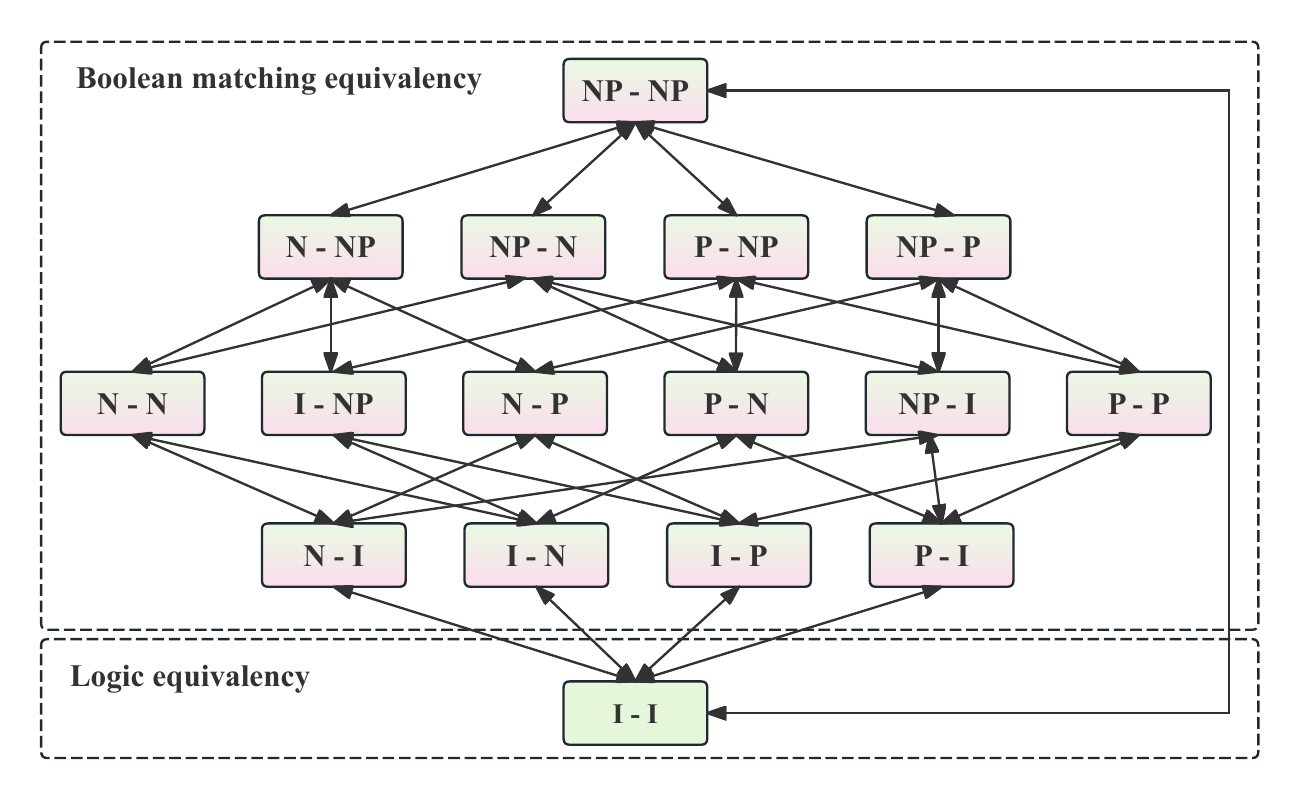}
    \caption{The domain relation of the logic equivalency and Boolean matching equivalency.}
    \vspace{-1.5em}
    \label{fig:domain}
\end{figure}

\subsubsection{Transformation of Boolean Matching.}
\begin{definition}[permutation]
The permutation~(\textbf{P}) function $\pi$ over a node set $\mathbf{S}~(\mathbf{S} \in \mathcal{V})$ is a bijection function that $\pi(\mathbf{S})=\mathbf{S}'$, which will lead to $\mathbf{S}$ and $\mathbf{S}'$ having different orders.
\label{def:perm}
\vspace{-1.5em}
\end{definition}

\begin{definition}[negation]
The negation~(\textbf{N}) function $\nu$ over a node set $\mathbf{S}~(\mathbf{S} \in \mathcal{V})$ is a componentwise mapping function that $\nu(\mathbf{S}_i) = \neg \mathbf{S}_i$, which will lead to the negation of the node.
\label{def:neg}
\vspace{-1.5em}
\end{definition}

\begin{definition}[Boolean matching]
\label{def:boolean_matching}
Given a pair of Boolean circuits ($\mathcal{C}_1$, $\mathcal{C}_2$), we can say that $\mathcal{C}_1$ is Boolean matched to $\mathcal{C}_2$ under the transformation of Boolean matching, denoted as $\mathcal{C}_1 \circeq \mathcal{C}_2$, and it can be formulated as follows:
\begin{equation}
\nonumber
\begin{aligned}
\exists \nu_{O} \circ \pi_{O}, \exists \nu_{I} \circ \pi_{I}, \forall \Vec{x}, \Vec{y}, & \bigg{(}\Vec{f}(\Vec{x}) \wedge \nu_{O} \circ \pi_{O}\Big{(}\Vec{g}\big{(}\nu_{I} \circ \pi_{I}(\Vec{y})\big{)}\Big{)}\bigg{)} \\
                                                 & \Rightarrow \bigwedge_{i=1}^{m}\Big{(}f_i(\Vec{x}) \equiv {g'}_i\big{(}\nu_{I} \circ \pi_{I}(\Vec{y})\big{)}\Big{)} \\
                                                 & \Rightarrow \bigwedge_{i=1}^{m}\Big{(}f_i(\Vec{x}) \equiv {g'}_i\big{(}\Vec{y'})\big{)}\Big{)}
\end{aligned}
\end{equation}
where $\nu_{I}$ and $\pi_{I}$ mean the transformation functions act on the input nodes~($\Vec{x}$ and $\Vec{y}$), while $\nu_{O}$ and $\pi_{O}$ mean on the output nodes~($\Vec{f}$ and $\Vec{g}$), $\circ$ means the joint action, and $g' = \nu_{O} \circ \pi_{O}(g)$.
\end{definition}

According to \cref{def:boolean_matching}, the Boolean matching~\cite{boolean_matching} process involves determining whether two Boolean Circuits are equivalent under the permutation or negation transformations.
It plays a critical role in tasks like logic optimization, technology mapping, and verification by confirming functional equivalency.
And \cref{fig:full_adder} (a)$\to$(c) and (b)$\to$(d) illustrate the result of the permutation transformation process.

\section{Analysis of the Boolean-aware Boolean Circuit Classification Problem}
\label{sec:problem}
In this section, we first give the domain of the ``Boolean-aware''.
Then, we analyze the laws and influence of the Boolean transformation under the given domain.
Finally, we give the details of the problem formulation.

\subsection{The domain of \textbf{\textit{Boolean-aware}}}
The domain of the ``Boolean-aware'' focuses on the Boolean transformation of logic optimization and Boolean matching.

\subsubsection{Domain relation}

\begin{definition}[identity]
The identity~(\textbf{I}) function $\iota$ refers to the logic equivalent functions over a node set $\mathbf{S}$($\mathbf{S} \in \mathcal{V}$).
\label{def:identity}
\end{definition}

\begin{remark}
The identity function $\iota$ is to cope with the permutation function $\pi$, and negation function $\nu$ that it has no influence on the inputs or outputs.
\end{remark}

The domain relation is based on the ``X-Y equivalence'' under the Boolean transformation for X, Y $\in$ \{I, N, P, NP\}, where X and Y denote equivalence conditions on the input and output sides, respectively, and I, N, P, and NP stand for identity, negation, permutation, and negation plus permutation, respectively.
According to \cref{fig:domain}, there are 16 ``X-Y'' equivalence of the domain of ``Boolean-aware''.
The ``I-I'' refers to the logic equivalence, and the other 15 types are the equivalence under Boolean matching.

\subsubsection{Complexity}
And for a certain $n$-input and $m$-output vectorial Boolean function $\Vec{f}$, the complexity of the negation $\nu$ and permutation $\pi$ are  $O(2^{n} \cdot 2^{m})$ and $O(n! \cdot m!)$ respectively.
Thus, the worst complexity of the worst condition~(NPNP) of Boolean matching is $O(2^{n} \cdot n! \cdot 2^{m} \cdot m!)$.

\subsection{Analysis on the laws}
\label{sec:problem:expand}

\subsubsection{The laws under logic optimization}


\begin{table}[t]
\centering
\caption{The laws under logic optimization}
\begin{tabular}{cc}
\toprule
\thead{Logic Equivalent Law} & \thead{Description}  \\
\toprule
Associative Law     &  \makecell{$A\wedge(B\wedge C) \Leftrightarrow (A\wedge B)\wedge C$,\\ $A\vee(B\vee C) \Leftrightarrow (A\vee B)\vee C$ }  \\ \hline
Distributive Law    &  \makecell{$A\wedge(B\vee C) \Leftrightarrow (A \wedge B) \vee (A \wedge C)$,\\ $A\vee(B\wedge C) \Leftrightarrow (A \vee B) \wedge (A \vee C)$ } \\ \hline
Double Negation Law &  \makecell{$\neg (\neg A) \Leftrightarrow A$ } \\ \hline
De Morgan's Law     &  \makecell{$ \neg(A\vee B)  \Leftrightarrow \neg A \wedge \neg{B} $,\\  $\neg(A\wedge B)  \Leftrightarrow \neg A \vee \neg{B}$ }\\ \hline
Zero Law            &  \makecell{$A\vee 0 \Leftrightarrow 0$, $A\wedge 0 \Leftrightarrow A$ } \\ \hline
Identity Law        &  \makecell{$A\vee 1 \Leftrightarrow A$, $A\wedge 1 \Leftrightarrow 1$ } \\ \hline
...   &  ... \\
\bottomrule
\end{tabular}
\label{tab:le_laws}
\vspace{-1.5em}
\end{table}

\begin{definition}[logic equivalent class]
Given any two Boolean circuits $\mathcal{C}_1$ and $\mathcal{C}_2$ in a logic equivalent class, which can derive $\mathcal{C}_1 \equiv \mathcal{C}_2$.
\label{def:le_class}
\end{definition}

The base principle for the logic optimization operations is based on the logic equivalent transformation.
In addition, the logic equivalent transformation is the combination of the logic equivalent laws as shown in \cref{tab:le_laws}.
According to the definition of ``logic equivalent'', the Boolean function can be defined as the identifier of a Boolean circuit.
The logic optimization operations can generate the logic equivalent classes as the statement of \cref{def:le_class}.

The \cref{fig:full_adder:a}$\leftrightarrow$\cref{fig:full_adder:b} shows the logic optimization by equivalent local replacement.
We can get ``$SUM=C_{in} \oplus (A \oplus B)$, $C_{out}=\big{(}A \wedge B\big{)} \vee \big{(}C_{in} \wedge (A \oplus B)\big{)}$'' in \cref{fig:full_adder:a}, and ``$SUM=C_{in} \oplus (A \oplus B)$, $C_{out}=\big{(}A \wedge B\big{)} \vee \big{(}C_{in} \wedge (A \vee B)\big{)}$'' in \cref{fig:full_adder:b}. And these two $C_{out}$ are logic equivalent through the logic equivalent laws in \cref{tab:le_laws}:
\begin{footnotesize}
\begin{equation}
\nonumber
\begin{aligned}
C_{out} &= \big{(}A \wedge B\big{)} \vee \big{(}C_{in} \wedge (A \vee B)\big{)} \\
        &= (A \wedge B) \vee (A \wedge C_{in}) \vee (B \wedge C_{in}) \\
        &= (A \wedge B) \vee (A \wedge C_{in}) \vee \big{(}B \wedge C_{in} \wedge (A \vee \neg{A})\big{)} \\
        &= (A \wedge B) \vee (A \wedge B \wedge C_{in}) \vee (A \wedge C_{in}) \vee (\neg{A} \wedge B \wedge C_{in}) \\
        &= \big{(}A \wedge B \wedge (1 \vee C_{in})\big{)} \vee (A \wedge C_{in}) \vee (\neg{A} \wedge B \wedge C_{in}) \\
        &= (A \wedge B) \vee (A \wedge C_{in}) \vee (\neg{A} \wedge B \wedge C_{in}) \\
        &= (A \wedge B) \vee \Big{(} C_{in} \wedge \big{(} (\neg{A} \wedge B) \vee (A \wedge \neg{B})\big{)} \Big{)}\\
        &= \big{(}A \wedge B\big{)} \vee \big{(}C_{in} \wedge (A \oplus B)\big{)} \\
\end{aligned}
\end{equation}
\end{footnotesize}
where the effect of the full adder in \cref{fig:full_adder:a} is that the gate $A \oplus B$ can be shared by $SUM$ and $C_{out}$.

\subsubsection{The laws under Boolean Matching}

\begin{table}[t]
\centering
\caption{The laws under Boolean matching.}
\begin{tabular}{cc}
\toprule
\thead{Boolean Matching Law} & \thead{Description~($\mathbf{S}$ is inputs/outputs node vector)}  \\
\toprule
Permutation Law  &  $\pi(\{\mathbf{S}\}) = \mathbf{S}'$ \\ \hline
Negation Law     &  $\nu(\mathbf{S}_i) = \neg \mathbf{S}_i$ \\ \hline
Invertible Law   &  \makecell{$\pi^{-1} \circ \pi = \iota$, \\ $\nu^{-1} \circ \nu = \iota$} \\
\bottomrule
\end{tabular}
\label{tab:bm_laws}
\vspace{-1.5em}
\end{table}

The laws under Boolean Matching mainly concentrate upon the negation function $\nu$ and permutation function $\pi$.
And \cref{tab:bm_laws} shows the three laws under Boolean matching.
Given an input/output nodes vector $\mathbf{S}$:
\begin{itemize}
    \item \textbf{permutation law}. According to \cref{def:perm}, we can get that ``$\pi(\{\mathbf{S}\}) = \mathbf{S}'$'' that $\mathbf{S}$ and $\mathcal{S}'$ have the same collection contents but different order~($\{\mathbf{S}\} == \{\mathbf{S}'\}$, yet $\exists i, \mathbf{S}_i \neq \mathbf{S}'_i$).
    \item \textbf{negation law}. The \cref{def:neg} shows that the element of $\mathbf{S}$ can be inverted that ``$\nu(\mathbf{S}_i) = \neg \mathbf{S}_i$''.
    \item \textbf{invertible law}. The invertible law is derived from the permutation and negation laws by their invertible property.
Thus, the inverse permutation function $\pi^{-1}$ and inverse negation function $\nu^{-1}$ can recover the original Boolean function by the operations $\pi^{-1} \circ \pi = \iota$ and $\nu^{-1} \circ \nu = \iota$~($\iota$ is defined at \cref{def:identity}).
The $\pi^{-1}$ function recovers the original order of $\mathbf{S}$ of the Boolean circuit, and $\nu^{-1}$ function recovers the phase or polarity of the nodes in $\mathbf{S}$.
\end{itemize}

\subsection{Analysis on the influence}
\label{sec:problem:influence}

\begin{figure}[t]\scriptsize
    \centering
    \subfigure[original full adder 1: \newline$SUM (A, B , C_{in}) = C_{in} \oplus (A \oplus B)$, \newline $C_{out}(A,B,C_{in})=\big{(}A \wedge B\big{)} \vee \big{(}C_{in} \wedge (A \oplus B)\big{)}$.]{ 
        \begin{minipage}[t]{0.22\textwidth}
        \centering
        \includegraphics[width=0.6\textwidth]{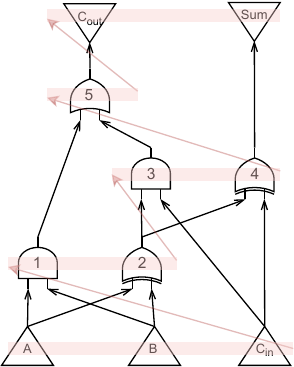}
        \end{minipage}
        \label{fig:full_adder:a}
    }
    \hfill
    \subfigure[original full adder 2: \newline$SUM(A,B,C_{in})=C_{in} \oplus (A \oplus B)$, \newline $C_{out}(A,B,C_{in})=\big{(}A \wedge B\big{)} \vee \big{(} C_{in} \wedge (A \vee B)\big{)}$.]{ 
        \begin{minipage}[t]{0.22\textwidth}
        \centering
        \includegraphics[width=0.6\textwidth]{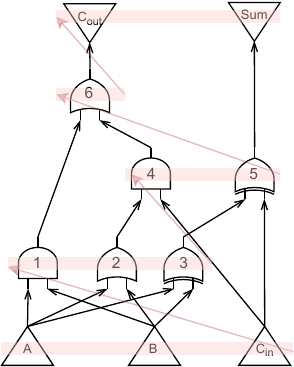}
        \end{minipage}
        \label{fig:full_adder:b}
    }
    \subfigure[permuted full adder 1: \newline$SUM(C_{in},A,B)=C_{in} \oplus (A \oplus B)$, \newline $C_{out}(C_{in},A,B)=\big{(}A \wedge B\big{)} \vee \big{(}C_{in} \wedge (A \oplus B)\big{)}$.]{ 
        \begin{minipage}[t]{0.22\textwidth}
        \centering
        \includegraphics[width=0.6\textwidth]{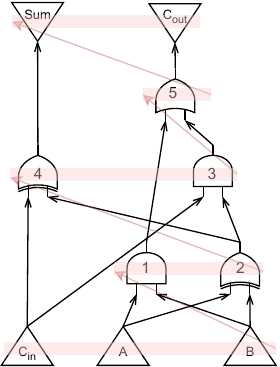}
        \end{minipage}
        \label{fig:full_adder:c}
    }
    \hfill
    \subfigure[permuted full adder 2: \newline$SUM(C_{in},A,B)=C_{in} \oplus (A \oplus B)$, \newline $C_{out}(C_{in},A,B)=\big{(}A \wedge B\big{)} \vee \big{(} C_{in} \wedge (A \vee B)\big{)}$.]{ 
        \begin{minipage}[t]{0.22\textwidth}
        \centering
        \includegraphics[width=0.6\textwidth]{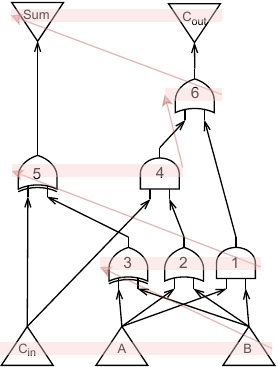}
        \end{minipage}
        \label{fig:full_adder:d}
    }
    \subfigure[negated full adder 1: \newline$SUM(A,B,C_{in})= \neg\big{(}C_{in} \oplus (A \oplus \neg{B})\big{)}$, \newline $C_{out}(A,B,C_{in})=\big{(}A \wedge \neg{B}\big{)} \vee \big{(}C_{in} \wedge (A \oplus \neg{B})\big{)}$.]{ 
        \begin{minipage}[t]{0.22\textwidth}
        \centering
        \includegraphics[width=0.6\textwidth]{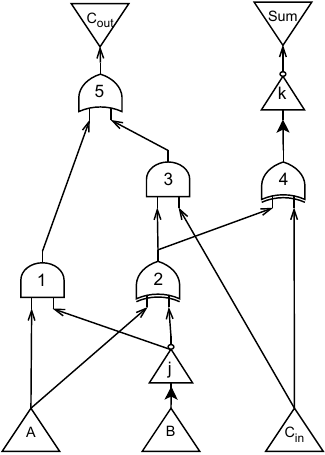}
        \end{minipage}
        \label{fig:full_adder:e}
    }
    \hfill
    \subfigure[negated full adder 2: \newline$SUM(A,B,C_{in})=\neg\big{(}C_{in} \oplus (A \oplus \neg{B})\big{)}$, \newline $C_{out}(A,B,C_{in})=\big{(}A \wedge \neg{B}\big{)} \vee \big{(} C_{in} \wedge (A \vee \neg{B})\big{)}$.]{ 
        \begin{minipage}[t]{0.22\textwidth}
        \centering
        \includegraphics[width=0.6\textwidth]{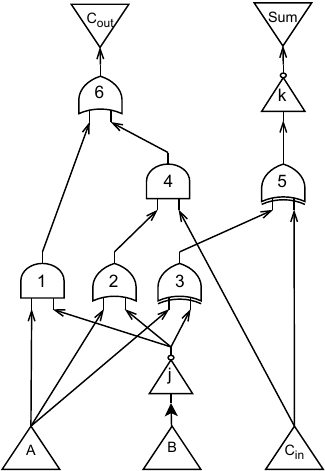}
        \end{minipage}
        \label{fig:full_adder:f}
    }
    \caption{The visualization of the full adder under the Boolean transformations. \textbf{Logic optimization}: (a)$\leftrightarrow$(b); \textbf{Permutation}: (a)$\leftrightarrow$(c), and (b)$\leftrightarrow$(d); \textbf{Negation}: (a)$\leftrightarrow$(e), and (b)$\leftrightarrow$(f); \textbf{Negation$\oplus$Permutation}: (c)$\leftrightarrow$(e), and (d)$\leftrightarrow$(f).}
    \label{fig:full_adder}
\end{figure}

\subsubsection{Boolean Matching change Boolean Circuits}

\paragraph{Permutation $\pi$ changes Topological Order.}
The permutation transformation $\pi$ on Boolean circuits is mainly implemented by changing the position of the input/output nodes.
Then, the topological order will also be permutated by the modification.
In the analysis of two commonly used topological order methods: depth-first-search~(DFS), and breadth-first-search~(BFS).

\begin{table}[t]
\centering
\caption{The topological order of Boolean circuits in \cref{fig:full_adder} by BFS method.}
\begin{tabular}{c|c}
\toprule
Circuits Location & Topological Order \\
\midrule
\cref{fig:full_adder:a} & $A, B, C_{in}, 1, 2, 3, 4, 5, C_{out}, SUM$ \\
\cref{fig:full_adder:c} & $C_{in}, A, B, 1, 2, 4, 3, 5, SUM, C_{out}$  \\
\cref{fig:full_adder:e} & $A, B, C_{in}, j, 1, 2, 3, 4, 5, k, C_{out}, SUM$  \\
\midrule
\cref{fig:full_adder:b} & $A, B, C_{in}, 1, 2, 3, 4, 5, 6, C_{out}, SUM$  \\
\cref{fig:full_adder:d} & $C_{in}, A, B, 3, 2, 1, 4, 5, 6, SUM, C_{out}$  \\
\cref{fig:full_adder:f} & $A, B, C_{in}, j, 1, 2, 4, 3, 5, 6, k, C_{out}, SUM$  \\
\bottomrule
\end{tabular}
\label{tab:bfs}
\end{table}

And \cref{fig:full_adder:a}$\leftrightarrow$\cref{fig:full_adder:c} and \cref{fig:full_adder:b}$\leftrightarrow$\cref{fig:full_adder:d} show the permutation result of the full adder.
The red line depicts the topological order for each full adder by the BFS algorithm.
The \cref{tab:bfs} shows the detailed topological orders in \cref{fig:full_adder} in the BFS method, and results based on DFS are also modified and different.

\paragraph{Negation $\nu$ modifies circuits.}
The negation transformation $\nu$ on Boolean circuits is mainly implemented by adding or removing the ``Inverter'' gate on PIs/POs.
In terms of the result, the Boolean Circuit's size or depth can be affected.
However, the relative position of all nodes in the original Boolean circuit in the topology sequence will not change for a deterministic topology ordering algorithm.

The \cref{fig:full_adder:a}$\leftrightarrow$\cref{fig:full_adder:e} and \cref{fig:full_adder:b}$\leftrightarrow$\cref{fig:full_adder:f} show the negation on the node $B$ and $SUM$.
From the direction of \cref{fig:full_adder:a}$\rightarrow$\cref{fig:full_adder:e}, the inverter nodes ``$j$ and $k$'' are added; while they are removed from the direction \cref{fig:full_adder:e}$\rightarrow$\cref{fig:full_adder:a}.
Furthermore, nodes apart from the added/removed inverter nodes will keep the original relative position in the topological order.

\subsubsection{Boolean Matching expands the equivalent class}

The combination of logic optimization and Boolean matching operations can further enlarge the search space of a Boolean circuit.
In the following, we will give the theorem about the Boolean matching expands the equivalent class space.

\begin{theorem}
The invertible law of Boolean Matching will allow the logic-nonequivalent operations of negation and permutation to be used in the logic-equivalent laws.
\label{thm:local_replacement}
\end{theorem}
\begin{proof}[Proof of \cref{thm:local_replacement}]
For any given Boolean circuit $\mathcal{C}$, its input nodes $I_S$ and output nodes $O_S$, it is trivial that logic optimization operations indeed generate the logic-equivalent class.
However, it will generate logic-nonequivalent classes for the Boolean matching operations.
For the negation operation: the logic of the input and output nodes will changed to $\nu_I(I_S)$ and $\nu_O(O_S)$.
However, we can restore or recomputed the invertible functions $\nu^{-1}_I$ and $\nu^{-1}_O$ to recover the original state by $\nu_I(I_S) \circ \nu^{-1}_I(I_S)$ and $\nu_O(O_S) \circ \nu^{-1}_O(O_S)$.
And it is similar for the permutation operation to recover the original state by using the $\pi^{-1}_I$ and $\pi^{-1}_O$.
\end{proof}

The \cref{thm:local_replacement} provides the theoretical support that the logic-nonequivalent operations of negation and permutation can be applied to the logic equivalent operations fields, such as logic optimization, technology mapping, engineer change order, \textit{etc}.
It allows the equivalent local replacement of the Boolean circuits with the equivalence under Boolean matching after applying the corresponding inverse functions.

\begin{figure}[t]
    \centering
    \subfigure[$f=x_1x_2 + x_3$.]{ 
        \begin{minipage}[t]{0.22\textwidth}
        \centering
        \includegraphics[width=0.78\textwidth]{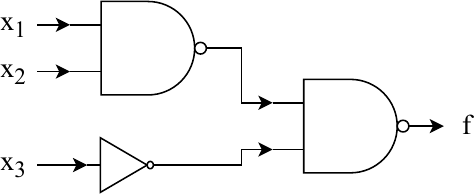}
        \end{minipage}
        \label{fig:npn:a}
    }
    \hfill
    \subfigure[$g=\neg{x_1} + x_2\neg{x_3}$.]{ 
        \begin{minipage}[t]{0.22\textwidth}
        \centering
        \includegraphics[width=0.98\textwidth]{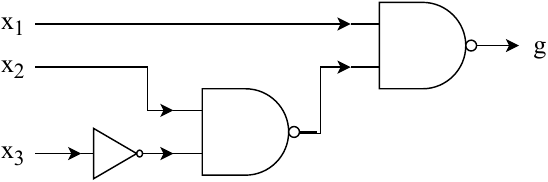}
        \end{minipage}
        \label{fig:npn:b}
    }
    \caption{Two NPN-equivalent functions $f$ and $g$.}
    \label{fig:npn}
    \vspace{-1em}
\end{figure}

And \cref{fig:npn} shows two Boolean matching functions $f$ and $g$ are NPN-equivalent under and permutation $\pi_I(x_1, x_2, x_3) = {x_3, x_2, x_1}$ and negation $\nu_I(x_1, x_2, x_3) = {\neg{x_1}, x_2, \neg{x_3}}$, thus, $g(x_1, x_2, x_3) = f(\neg{x_3}, x_2, \neg{x_1})$.

Hence, the Boolean circuits under the Boolean matching operations can also be logic-equivalent according to \cref{thm:local_replacement}.
It enlarges the equivalent class.
Given an $n$-inputs and $m$-outputs Boolean circuit $C$, and assuming that the size of its logic equivalent class after logic optimization operations is $N$,
the space can be expanded to $(2^n \cdot n! \cdot 2^m \cdot m!) \cdot N$ at most.
In the meantime, the transformed Boolean circuit by Boolean matching can also apply the logic optimization operations for other design space exploration.

\subsection{Problem Formulation}

\begin{definition}[Matching Equivalent Class]
The matching equivalent class of Boolean circuits refers to the group of Boolean circuits that any given two Boolean circuits $\mathcal{C}_1$ and $\mathcal{C}_2$ can be logic equivalent under logic optimization and Boolean matching transformations.
And \cref{fig:boolean_transformation} and \cref{fig:domain} give the domain and illustration of the matching equivalent class.
\label{def:class}
\end{definition}

\begin{problem}[Boolean-aware Boolean Circuits Classification]
\label{prob:booleanClassification}
Given a set of Boolean circuits $\{\mathcal{C}_i\}_{i=1}^n$, classify the circuits into $k~(k\le n)$ classes, and each class follows the property of the defined matching equivalent class.
\end{problem}

\begin{theorem}
The matching equivalent class is closed under the logic equivalent and Boolean matching laws.
\label{thm:closure}
\end{theorem}

\begin{proof}[Proof of \cref{thm:closure}]
For any given Boolean circuit $\mathcal{C}_1$ and the matching equivalent class $\mathbf{S}$~($\mathcal{C}_1 \in \mathbf{S}$): 
1) Closure under logic equivalence: for another Boolean circuit $\mathcal{C}_2$ that $\mathcal{C}_1 \overset{logic\ optimization}{\rightarrow} \mathcal{C}_2$ with $\mathcal{C}^{F}_1 \equiv \mathcal{C}^{F}_2$, then, $\mathcal{C}_2 \in \mathbf{S}$;
2) Closure under Boolean matching laws: for another Boolean circuit $\mathcal{C}_2$ that $\mathcal{C}_1 \overset{logic\ optimization}{\rightarrow} \mathcal{C}_2 \overset{\nu \circ \pi}{\rightarrow} \mathcal{C}_3$ with $\mathcal{C}^{F}_1 \equiv \mathcal{C}^{F}_2~and~\mathcal{C}^{F}_2 \overset{\nu^{-1} \circ \pi^{-1}}{=} \mathcal{C}^{F}_3$, then $\mathcal{C}_3 \in \mathbf{S}$.
Therefore, the \cref{thm:closure} is proved.
\end{proof}

\subsection{Challenges}
\paragraph{The Disturbance by Boolean equivalence.}
Boolean-equivalence asserts Boolean circuits with the same Boolean function can have different graph structures.
The depth and size all can differ from each other.
This brings the challenge of learning a consistent representation of different structures.

\paragraph{The modification by Boolean matching operations.}
Then permutation $\pi$ can disturb the topological order of the Boolean circuit, while the negation $\nu$ changes the Boolean function.
This brings the challenge of considering $\pi$ and $\nu$ for the classification.



\begin{figure}[t]
    \centering
    \includegraphics[width=0.5\textwidth]{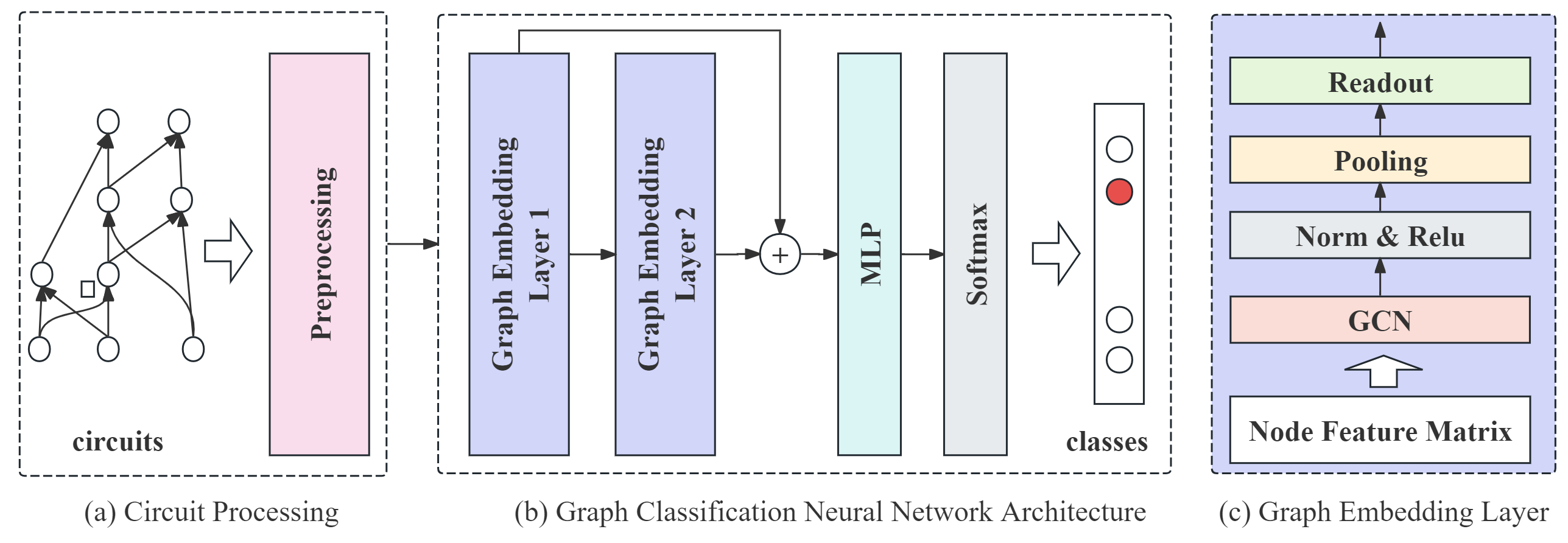}
    \caption{The proposed Boolean Circuit Classification Framework.}
    \label{fig:framework}
\end{figure}

\section{The Study Framework}
\label{sec:study}
In this section, we give a comprehensive study of the proposed problem on the graph neural network according to \cref{sec:problem}.

\subsection{Overview}
The \cref{fig:framework} depicts the proposed framework to solve the Boolean-aware Boolean circuit classification problem.
It mainly consists the two parts: (1) Boolean circuit processing; and (2) Graph Classification based on a graph neural network.
The following introductions are all focused on the alignment in the processing of Boolean-aware Boolean circuit classification.
We will not delve into the intricacies of GNN depth; instead, we will focus on the standard two-layer GNN model.


\subsection{The Power of Boolean Circuit Preprocessing}
\label{sec:study:preprocessing}

\begin{figure}[t]
    \centering
\vspace{-1em}
    \includegraphics[width=0.4\textwidth]{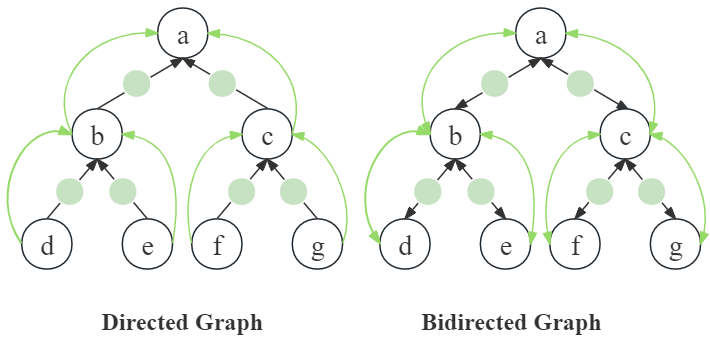}
    \caption{Aggreation direction on the two graph types.}
    \label{fig:graph_type}
    \vspace{-1em}
\end{figure}

\begin{figure}[t]
\centering
    \subfigure[Sub-circuit with Inverter.]{ 
        \begin{minipage}[t]{0.22\textwidth}
        \centering
        \includegraphics[width=0.4\textwidth]{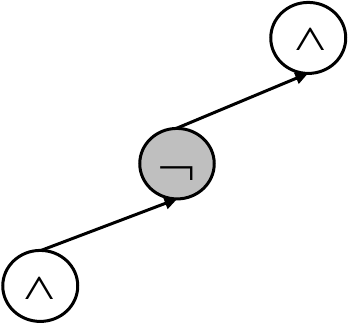}
        \end{minipage}
    }
    \hfill
    \subfigure[Sub-circuit without Inverter.]{ 
        \begin{minipage}[t]{0.22\textwidth}
        \centering
        \includegraphics[width=0.4\textwidth]{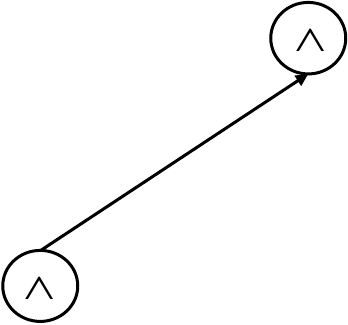}
        \end{minipage}
    }
    \caption{The illustrations that dropping the 2-degree nodes will not change the topology.}
    \label{fig:drop_inverter}
\end{figure}

In general, the Boolean circuit preprocessing step can produce graphs that are easier to handle.
It may greatly improve the accuracy and efficiency of the given task.
And we can list some techniques for the preprocessing:
\begin{itemize}
    \item The \cref{fig:graph_type} illustrates the directed Boolean circuit's graph and its bidirectional version. The direction of the edge can affect the node-embedding of the graph for graph learning tasks~(\cref{sec:study:GNN}). The bidigraph allows the aggregation of the output's direction for fanouts and siblings, while the digraph requires more on the GCN architecture designs.
    \item Dropping the 2-degree nodes. According to the graph theory~\cite{graph_degree_2005}, the 2-degree nodes can not affect the topology of graphs, and \cref{fig:drop_inverter} illustrates this case.
\end{itemize}

\subsection{The Power of the Graph Neural Network~(GNN)}
\label{sec:study:GNN}
The objective of GNN is to learn node features $\mathbf{X}$ by a permutation invariant function $\phi(\mathbf{X}_u, \mathbf{X}_{\mathcal{N}_{u}})$ over local neighborhoods on the adjacency matrix $\mathbf{A}$ for its downstream tasks.

\subsubsection{Node Embedding}

The node embedding aims at capturing and representing the feature information of nodes within a graph.
It uses the following three approaches to aggregate each node feature: 
1) \textbf{convolutional}, where sender node features are multiplied with a constant:
$h_u = \phi{\big{(} \mathbf{X}_u, \underset{v\in \mathcal{V}}{\oplus} c_{uv}\psi(\mathbf{X}_v) \big{)}}$;
2) \textbf{attentional}, where this multiplier is implicitly computed via an attention mechanism of the receiver over the sender:
$h_u = \phi{\big{(} \mathbf{X}_u, \underset{v\in \mathcal{V}}{\oplus} \alpha(\mathbf{X}_u, \mathbf{X}_v)\psi(\mathbf{X}_v) \big{)}}$;
and 3) \textbf{message-passing}, where vector-based messages are computed based on both the sender and receiver:
$h_u = \phi{\big{(} \mathbf{X}_u, \underset{v\in \mathcal{V}}{\oplus} \psi(\mathbf{X}_u, \mathbf{X}_v) \big{)}}$.
Typically, $\phi$ and $\psi$ are learnable, whereas $\oplus$ is realized as a nonparametric operation such as sum and mean.

\subsubsection{Graph Pooling~(Graph Reduction)}
Graph pooling, also known as graph reduction, refers to the process of reducing the size of the graph by selectively retaining ``important" nodes and their connections. This is critical in making graph neural networks (GNNs) more manageable and computationally efficient, especially for large graphs. Common graph pooling techniques include Top-K-based pooling, which selects nodes based on their scores computed by a learnable neural network, and cluster-based pooling, which groups nodes into clusters to form a coarsened graph.

\subsubsection{Graph Embedding~(Readout) \& Classification}
The graph embedding or readout layer is used to generate a whole-graph representation from node features, which is essential for tasks that require a global graph-level output, such as graph classification or graph regression. This is typically achieved through aggregation operations such as summing, averaging, or taking the maximum of node features across the entire graph. Advanced techniques might incorporate attention mechanisms to weigh node contributions dynamically.

The classification stage typically follows the graph embedding stage and uses the graph-level representation to classify the entire graph into one of several categories. This is usually implemented using a multi-layer perceptron (MLP) followed by a softmax function. The MLP captures non-linear relationships in the graph representation, and the softmax function outputs the probabilities of the graph belonging to each class.

\begin{table}[t]
\centering
\vspace{-1em}
\footnotesize
\caption{The characteristics of the selected data.}
\begin{tabular}{c|cccc}
\toprule
\textbf{Design} & \#\textbf{PIs} & \#\textbf{POs} & \#\textbf{Nodes} & \#\textbf{Depth} \\
\midrule
\textit{ctrl}      &    7     &    26    &    174    &    10          \\
\textit{comp}      &    32    &    3     &    173    &    24          \\
\textit{unreg}     &    36    &    16    &    158    &    8           \\
\textit{stepper\_motor\_drive}&    28    &    27     &    180    &    10        \\
\textit{count}     &    35    &    16    &    192    &    26          \\
\textit{s510}      &    25    &    15    &    221    &    10          \\
\textit{int2float} &    11    &    7     &    260    &    16          \\
\textit{router}    &    60    &    30    &    257    &    54          \\
\textit{cht}       &    47    &    36    &    271    &    9           \\
\textit{ttt2}      &    24    &    21    &    351    &    14          \\
\bottomrule
\end{tabular}
\label{tab:datasets}
\vspace{-1em}
\end{table}

\subsection{The effort on the proposed problem}
Benefiting from the theoretical permutation-invariant operations in the node embedding step of GNN, the permutation function $\pi$ on Boolean circuits can be well studied by this architecture according to \cref{def:perm}.
The following graph embedding and MLP-based classification steps are also less affected.
While the negation function $\nu$ can affect the node embedding results.
Thus, the N-related operations in \cref{fig:domain} need further study and learning by preprocessing or new negation-aware node embedding algorithms.

\section{Experiments}
\label{sec:experiment}

\subsection{Data preparation}

\begin{figure}[t]
    \centering
    \vspace{-1em}
    \includegraphics[width=0.5\textwidth]{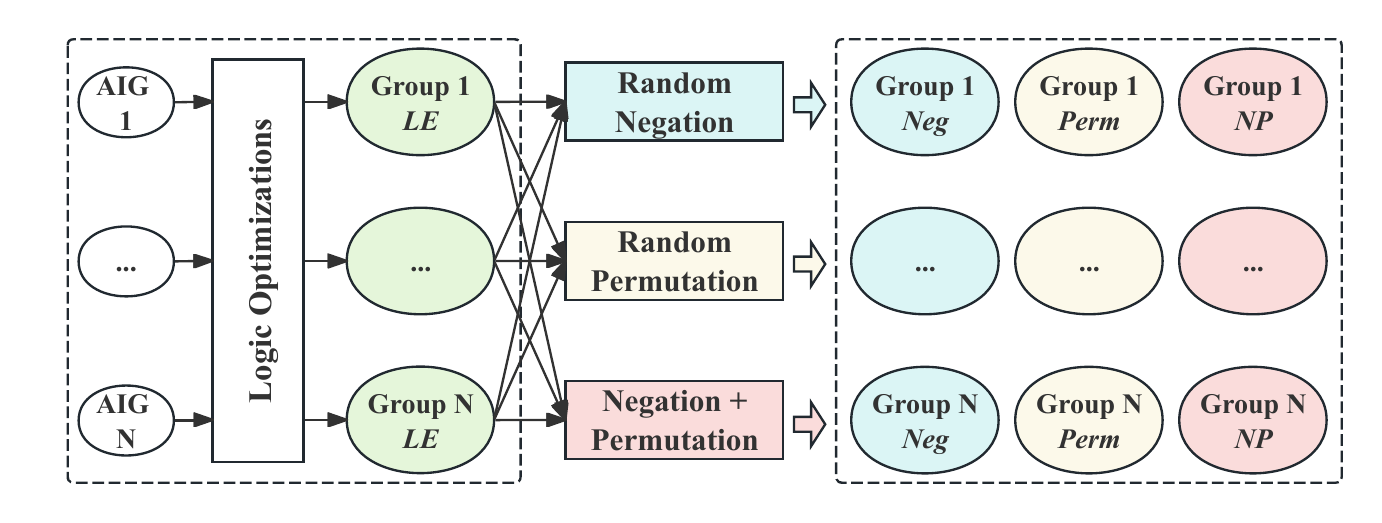}
    \caption{The illustration for the dataset generation flow. The ``LE'', ``Neg'', ``Perm'', and ``NP'' refer to the logic equivalent group, negation group, permutation group, and negative $\oplus$ permutation group, respectively. }
    \label{fig:data_flow}
\end{figure}

\begin{figure}[t]
    \centering
    \vspace{-1em}
    \includegraphics[width=0.4\textwidth]{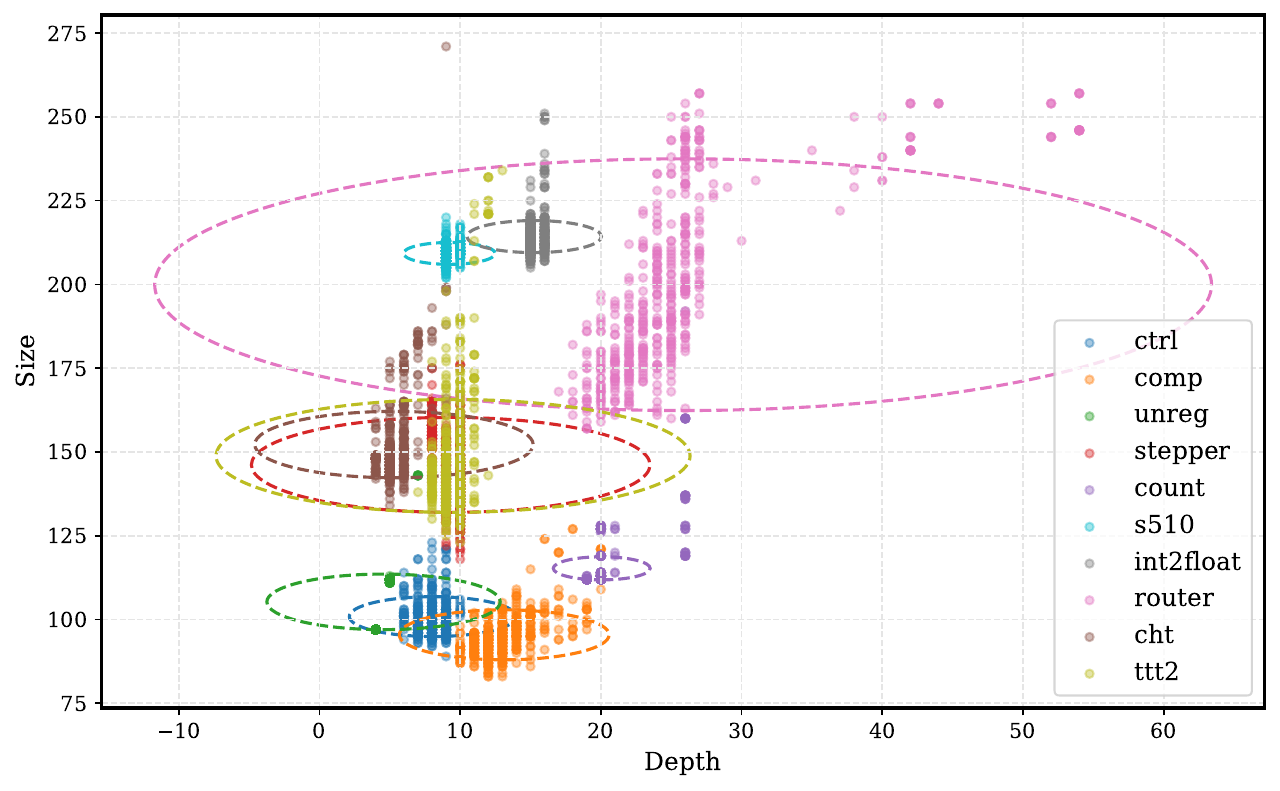}
    \caption{Size and depth distribution of the synthesised LE~(Logic Equivalent) groups. Each circle covers the 80\% of the designs of each group.}
    \label{fig:group}
    \vspace{-1em}
\end{figure}

The dataset\footnote{Compared with OpenABC-D~\cite{openabcd_animesh21}, the selected dataset is more closed.} was selected from multiple open combinational Boolean circuit benchmarks such as ISCAS85~\cite{ISCAS85}, ISCAS89~\cite{ISCAS89}, MCNC91~\cite{MCNC91}, IWLS2005~\cite{IWLS2005}, and EPFL~\cite{IWLS2015}. 
All of the data is in AIG format.
And \cref{tab:datasets} lists the detailed information on the source open benchmark of the Boolean circuit in AIG format.

The dataset generation flow is illustrated in \cref{fig:data_flow}. 
For each design in AIG format, the logic-equivalent class ``Group" is generated first through random\footnote{The ``random'' means the different optimization operators and length.} logic optimization. 
Next, for each AIG in ``Group", the AIGs with negation, permutation, and negation plus permutation are transformed and stored in their corresponding groups.
And \cref{fig:group} shows the size and depth distribution of the generated logic equivalent groups for the logic optimization step.
It reveals the overlap between different design groups.


\subsection{Environment and Configuration}


The code is written in Python language with the Pytorch package.
All procedures run on an Intel(R) Xeon(R) Platinum 8260 CPU with 2.40GHz, 24 cores, and 128GB RAM, NVIDIA A100 GPU.
The hyperparameters of GNN are as follows:
\textit{feature dimension}: 10;
\textit{hidden dimension}: 64;
\textit{learning rate}: 0.001;
\textit{weight decay}: 1e-5;
\textit{batch size}: 16.
We use the cross entropy loss function for this multi-classification task.

\subsection{Results Evaluation}
All the evaluations in this subsection are run for 100 epochs.

\subsubsection{Node embedding capacity of the Boolean circuit graph type}

\begin{figure}[t]
    \centering
    \includegraphics[width=0.45\textwidth]{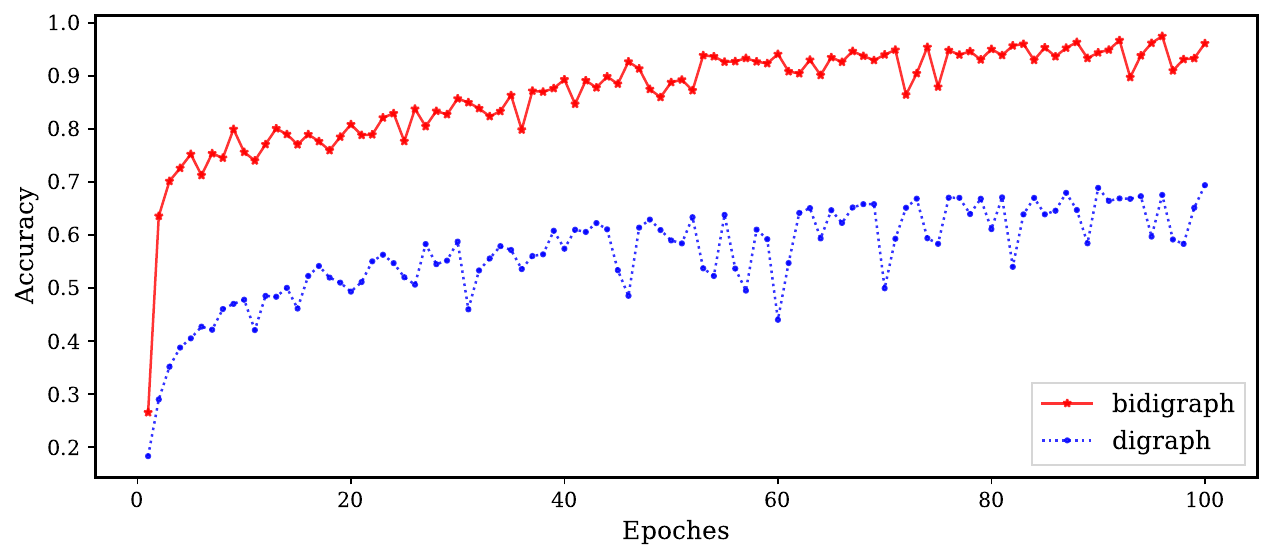}
    \caption{The accuracy curves for the capacity comparison of the digraph and bidigraph for the node embedding on the ``LE'' dataset by GCN model.}
    \label{fig:eval_graph}
\end{figure}

The \cref{fig:eval_graph} illustrates the accuracy curves of the digraph and bidigraph on the GCN architecture.
It shows that even the Boolean circuit is a directed graph, however, the bidirectional modification can also let the nodes capture more information for its sibling nodes of fanout nodes.

\subsubsection{Graph Pooling technology improve Semantic abstraction}

\begin{figure}[t]
    \centering
    \includegraphics[width=0.45\textwidth]{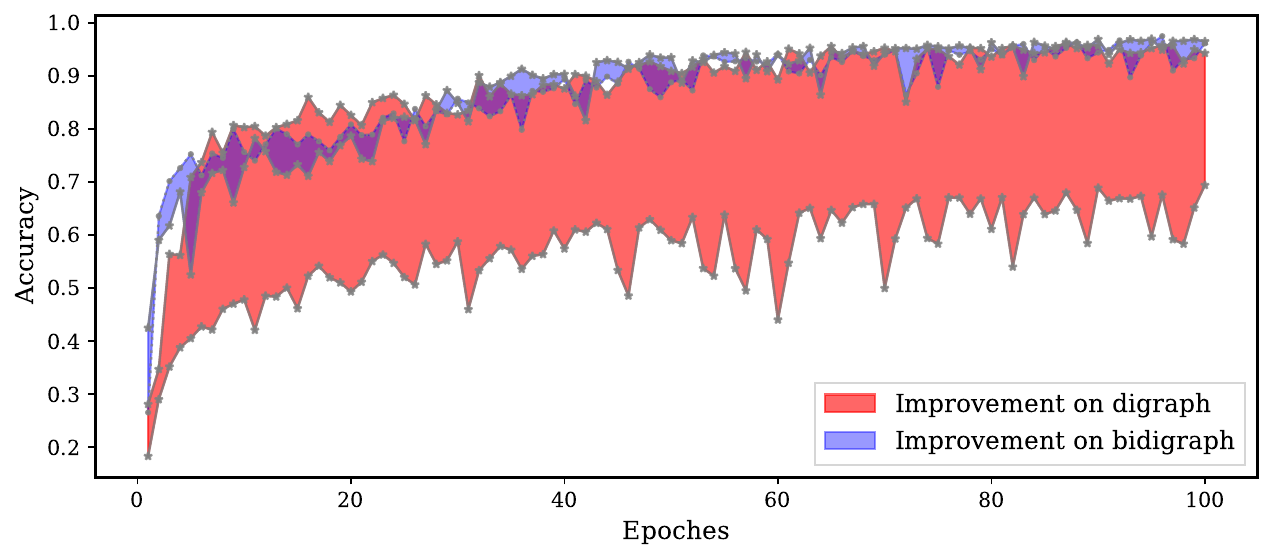}
    \caption{The accuracy curves for the capacity comparison of GCN and GCN+Pooling.}
    \label{fig:eval_pooling}
\vspace{-0.5cm}
\end{figure}
The \cref{fig:eval_pooling} shows the improvements of the graph pooling technology on the ``LE'' dataset, and the red and blue covered area are the improvement on digraph and bidigraph, respectively.
With the integrated graph pooling layer, the more abstracted semantic information can be captured for the Boolean circuit.

\subsubsection{Permutation V.S. Negation}

\begin{figure}[h]
    \centering
    \vspace{-1em}
    \includegraphics[width=0.45\textwidth]{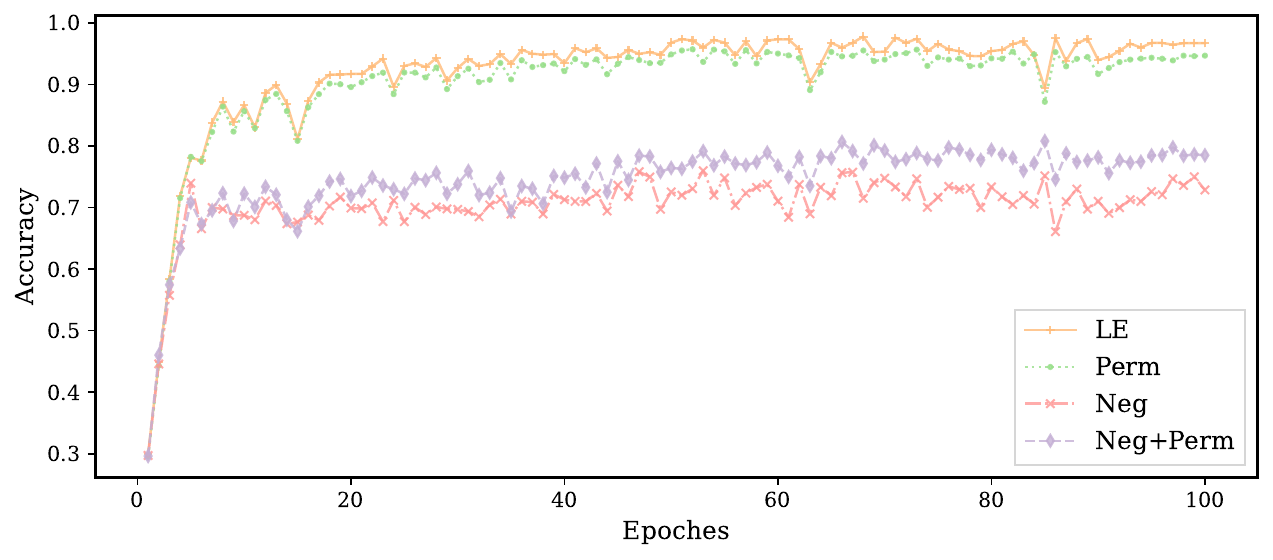}
    \caption{The accuracy curves comparison between the ``LE'', ``Neg'', ``Perm'', and ``Neg+Perm'' on the ``LE'' training dataset.}
    \label{fig:eval_np}
\vspace{-0.5cm}
\end{figure}

As the statement of the permutation-invariant operations in node embedding, the \cref{fig:eval_np} shows the accuracy curves on the four types of evaluation dataset ``LE'', ``Perm'', ``Neg'', and ``Neg+Perm'' with the ``LE'' training dataset.
It reveals that the evaluation of the permutation dataset has a similar curve compared with the ``LE'' dataset, while the curves on ``Neg'' and ``Neg+Perm'' behaviors are different and even poorer than ``Perm''.
The above results are also echoed in the previous analysis in \cref{sec:problem:influence}.

\begin{table}[t]\scriptsize
    \centering
    \caption{Accuracy on the bidigraph for the four types dataset.}
    \resizebox{0.45\textwidth}{8mm}{
    \begin{tabular}{l|cccc}
    \toprule
    \multicolumn{1}{c|}{\multirow{2}{*}{Type}} & \multicolumn{4}{c}{Accuracy}                     \\
    \multicolumn{1}{l|}{}                      & LE                  & Neg          & Perm     & Neg$\oplus$Perm   \\
    \midrule
    with inverter                              &  94.22($\pm{2.31}$) & 80.6($\pm{1.15}$) &  92.29($\pm{1.95}$) &  83.73($\pm{1.16}$)    \\
    without inverter                           &  99.14($\pm{0.55}$) & 98.76($\pm{0.65}$) &  98.75($\pm{0.7}$) &  98.74($\pm{0.75}$)    \\
    \bottomrule
    \end{tabular}}
    \label{tab:inverter}
\vspace{-0.5cm}
\end{table}

\begin{figure}[t]
    \centering
    \vspace{-2em}
    \subfigure[Digraph \textbackslash w inverter with the ``LE'' training dataset on GCN: accuracy curve.]{ 
        \begin{minipage}[t]{0.22\textwidth}
        \centering
        \includegraphics[width=1\textwidth]{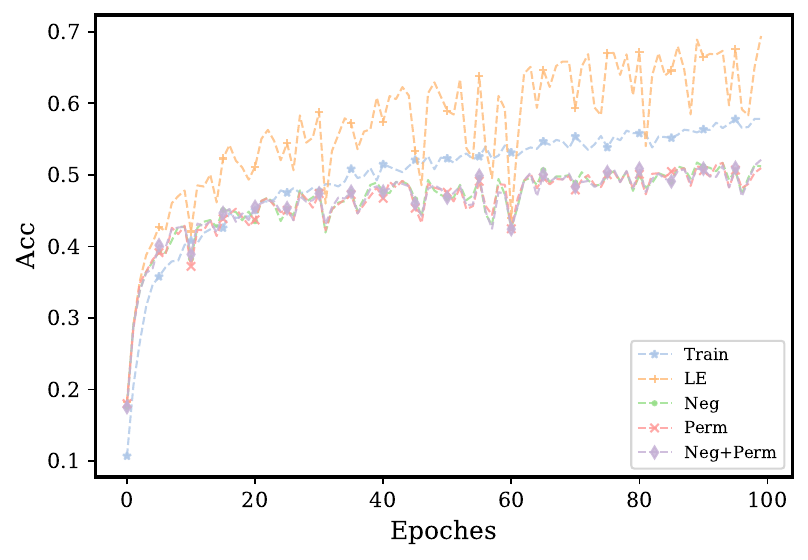}
        \end{minipage}
    }
    \vspace{-1em}
    \hfill
    \subfigure[Bidigraph \textbackslash wo inverter with the ``LE''+``Neg'' training dataset on GCN+Pooling: accuracy curve.]{ 
        \begin{minipage}[t]{0.22\textwidth}
        \centering
        \includegraphics[width=1\textwidth]{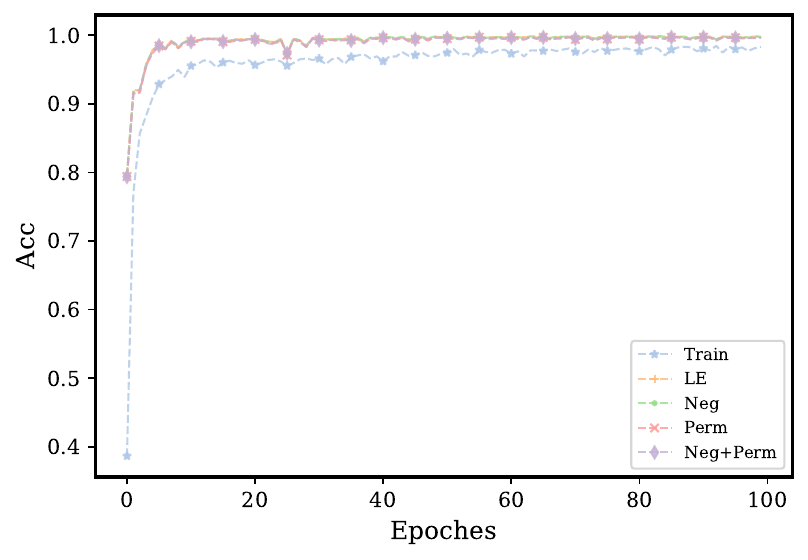}
        \end{minipage}
    }
    \vspace{-1em}
    \subfigure[Digraph \textbackslash w inverter with the ``LE'' training dataset on GCN: t-SNE visualization of ``Neg+Perm'' dataset.]{ 
        \begin{minipage}[t]{0.22\textwidth}
        \centering
        \includegraphics[width=1.1\textwidth]{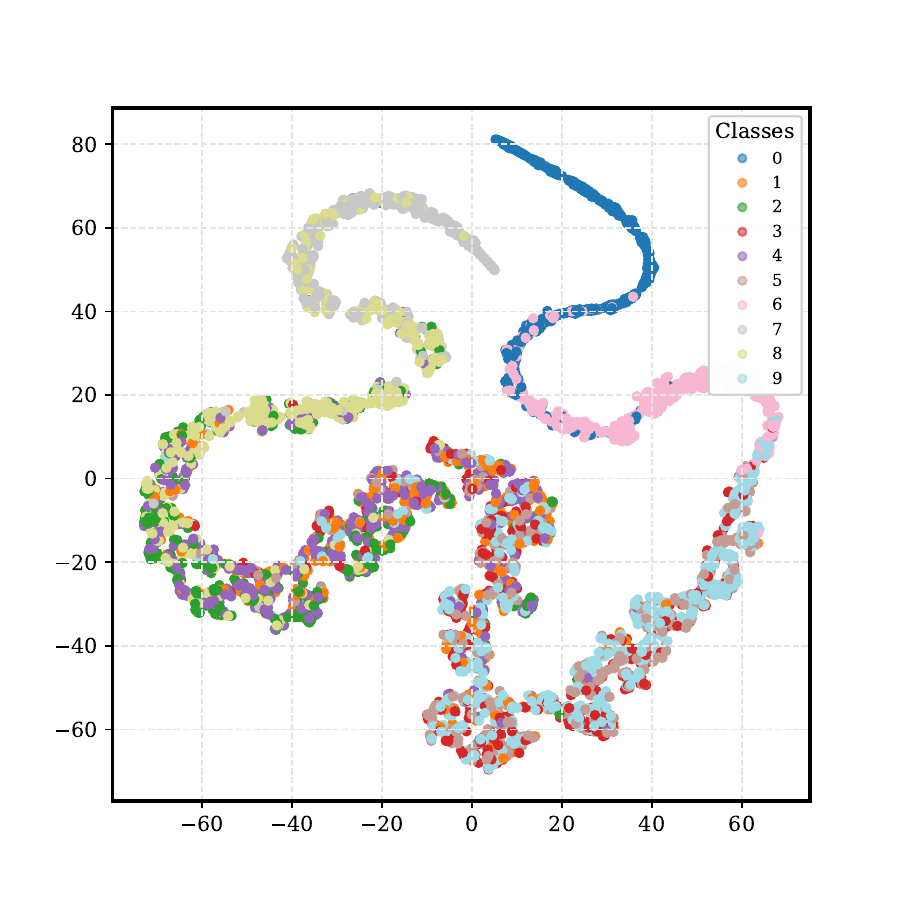}
        \end{minipage}
    }
    \hfill
    \subfigure[Bidigraph \textbackslash wo inverter with the ``LE''+``Neg'' training dataseton GCN+Pooling: t-SNE visualization of ``Neg+Perm'' dataset.]{ 
        \begin{minipage}[t]{0.22\textwidth}
        \centering
        \includegraphics[width=1.1\textwidth]{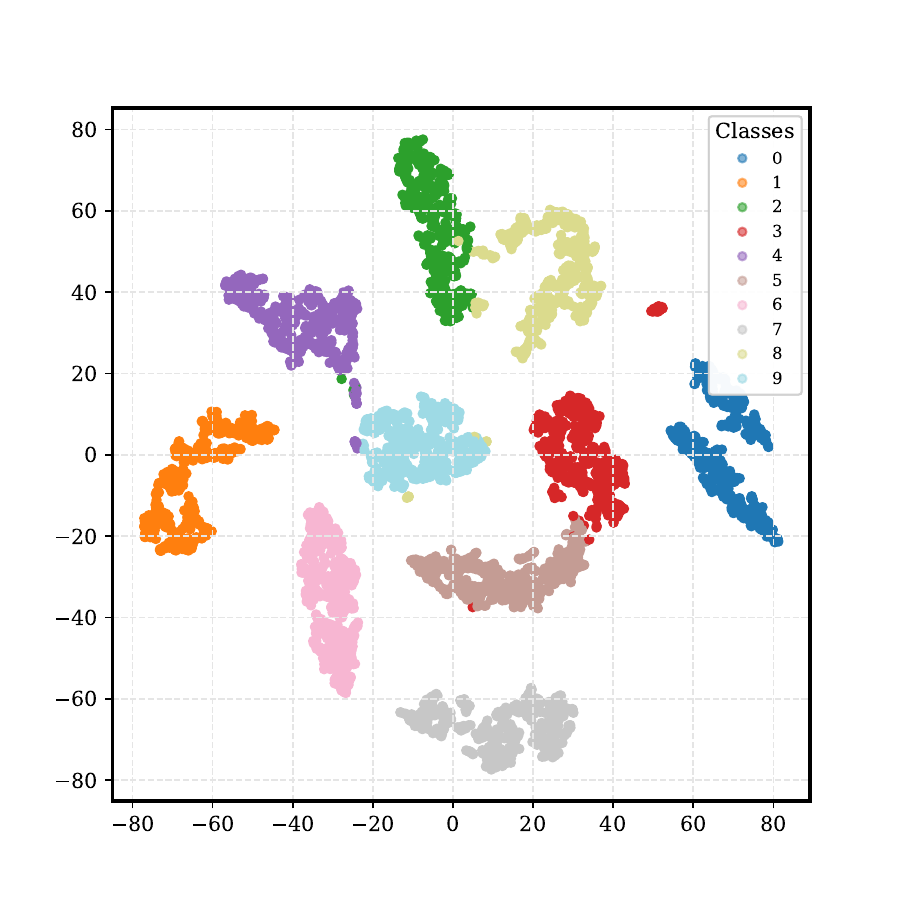}
        \end{minipage}
    }
    \caption{The visualization of the accuracy curve and t-SNE two extreme conditions.}
    \label{fig:eval_two_condition}
\end{figure}

Thus, the key factor is to solve the influence from the negation operation $\nu$.
The \cref{tab:inverter} shows the result for the inverter removal solution by the preprocessing in \cref{sec:study:preprocessing} with the training dataset on ``LE''.
It indicates that it is possible to align the four types of datasets by removing the inverter.

\subsubsection{Comparison of two extreme conditions\protect\footnote{Due to the page limit, we only list the two extreme cases here.}}

The \cref{fig:eval_two_condition} shows the two extreme conditions:
(1) \textbf{Cond1}: \cref{fig:eval_two_condition}(a) + \cref{fig:eval_two_condition}(c): it shows the classification result without any decoration on the GCN model; 
(2) \textbf{Cond2}: \cref{fig:eval_two_condition}(b) + \cref{fig:eval_two_condition}(d): it shows the classification result with the proposed preprocessing on the GCN + Pooling model.
The comparison of the accuracy curves shows the proposed tricks are useful in improving the accuracy and alignment for the proposed matching equivalent class.
The t-SNE figures on the ``Neg+Perm'' dataset also visualize the classification result.

\section{Applications}
\label{sec:application}




\begin{figure}[t]
    \vspace{-0.5cm}
    \centering
    \includegraphics[width=0.45\textwidth]{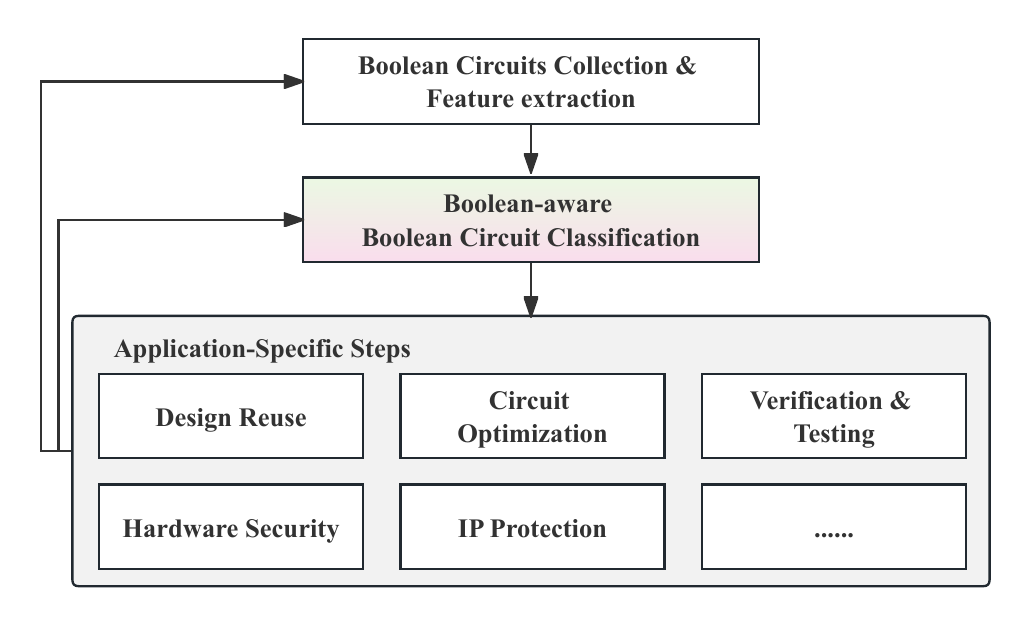}
    \caption{The specific applications.}
    \label{fig:application}
    \vspace{-0.5cm}
\end{figure}

The \cref{fig:application} shows the possible applications related to the proposed Boolean-aware Boolean circuit classification problem.
The proposed solution can enhance these applications more than the previous logic equivalent class.
As for a given application in \cref{fig:application}, the first step is to collect the task-related Boolean circuits and extract the features, then the Boolean-aware classification algorithm generates solutions for the downstream tasks as follows:
\begin{itemize}
    \item Design Reuse.
    \textbf{Purpose}: Facilitating the reuse of proven circuit designs in new projects;
    \textbf{Approach}:  Classify circuits to create a repository of reusable designs.
    \item Circuit Optimization. 
    \textbf{Purpose}: Identifying and replacing inefficient circuit designs; 
    \textbf{Approach}: Classify existing circuit designs based on efficiency metrics.
    \item Verification \& Testing. 
    \textbf{Purpose}: Enhancing the verification and testing processes for digital circuits; 
    \textbf{Approach}:  Classify circuits to categorize them based on testability and verification status.
    \item Hardware Security. 
    \textbf{Purpose}: Detecting and mitigating hardware Trojans; 
    \textbf{Approach}: Classify circuits to identify deviations from known good designs.
    \item IP Protection. 
    \textbf{Purpose}: Protecting intellectual property by identifying unauthorized use of proprietary circuits; 
    \textbf{Approach}: Classify proprietary circuits and store their signatures.
    \item ...
\end{itemize}

We demonstrate that all these applications can get a good initial solution by the Boolean-aware classification approach.

\section{Conclusion}
\label{sec:conclusion}
In this paper, we introduced a novel approach to Boolean circuit classification that leverages Boolean-aware methodologies.
We first proposed the \textit{matching equivalent class} and its corresponding classification problem based on the commonly used Boolean transformation in logic synthesis.
Then, a common study framework based on GNN is presented to analyze the improvement by the graph manipulation technology.
The experimental results echo the proposed theorems and analysis. 

Future work will investigate the applicability of our method to other domains where Boolean classification plays a critical role.
Overall, our proposed Boolean-aware Boolean circuit classification study framework provides a promising direction for enhancing the reliability and performance of Boolean circuit analysis.

\bibliographystyle{IEEEtran}
\bibliography{classifier}

\end{CJK}
\end{document}